\documentclass{article}

\usepackage{microtype}
\usepackage{graphicx}
\usepackage{subcaption}
\usepackage{booktabs}
\usepackage{hyperref}

\usepackage[accepted]{icml2026_fogen}

\usepackage{amsmath}
\usepackage{amssymb}
\usepackage{mathtools}
\usepackage{amsthm}
\usepackage{bbm}

\usepackage[capitalize,noabbrev]{cleveref}

\theoremstyle{plain}
\newtheorem{theorem}{Theorem}[section]

\newtheorem{corollary}[theorem]{Corollary}
\theoremstyle{definition}
\newtheorem{definition}[theorem]{Definition}

\theoremstyle{remark}
\newtheorem{remark}[theorem]{Remark}

\DeclareMathOperator{\Var}{Var}
\DeclareMathOperator{\Cov}{Cov}
\DeclareMathOperator{\Bias}{Bias}

\newcommand{\E}{\mathbb{E}}
\newcommand{\R}{\mathbb{R}}
\newcommand{\Prob}{\mathbb{P}}

\icmltitlerunning{Policy Gradient Foundations of GRPO: Credit Assignment, Gradient Sparsity, and Rank Collapse}

\begin{document}

\twocolumn[
\icmltitle{On the Policy Gradient Foundations of Group Relative Policy Optimization: \\
Credit Assignment, Gradient Sparsity, and Rank Collapse}

\icmlsetsymbol{equal}{*}

\begin{icmlauthorlist}
\icmlauthor{Amritansh Mishra}{aff1}
\icmlauthor{Supriyo Chakraborty}{aff1}
\icmlauthor{Berkcan Kapusuzoglu}{aff1}
\end{icmlauthorlist}

\icmlaffiliation{aff1}{Capital One}

\icmlcorrespondingauthor{Amritansh Mishra}{amritansh.mishra@capitalone.com}

\icmlkeywords{GRPO, Policy Gradient, RLHF, Credit Assignment, Gradient Sparsity}

\vskip 0.3in
]

\printAffiliationsAndNotice{}

\begin{abstract}
Group Relative Policy Optimization (GRPO) eliminates the learned critic in PPO by using the mean reward of grouped rollouts as a baseline. We provide a rigorous derivation of GRPO from first principles of the policy gradient theorem, revealing a fundamental credit assignment failure: under output-only reward, every token in a rollout receives identical advantage, collapsing token-level credit to a single scalar. We prove this induces gradient sparsity that intensifies over training, and demonstrate empirically via SVD analysis of GRPO gradients on Nemotron-4B/GSM8K that the gradient matrix has effective rank $\approx 2$ regardless of group size $R \in \{2,4,8\}$. We formalize this as an intrinsic rank-2 structure arising from the zero-sum constraint on advantages and derive conditions under which GRPO's baseline is optimal. Our results characterize when GRPO's simplicity is theoretically justified and identify the credit assignment bottleneck as the key limitation for multi-step reasoning.
\end{abstract}

\section{Introduction}
\label{sec:intro}

Policy gradient methods have become central to post-training of large language models, both for aligning with human preferences \cite{ouyang2022training} and for improving reasoning capabilities \cite{deepseekr1,shao2024deepseekmath}. The standard approach uses Proximal Policy Optimization (PPO) \cite{schulman2017proximal}, which requires a learned value function (critic) to estimate advantages doubling memory and compute footprint.

Group Relative Policy Optimization (GRPO) \cite{shao2024deepseekmath} offers an elegant alternative for each prompt, sample $R$ completions, compute their rewards, and use the group mean as the baseline. This eliminates the critic entirely. GRPO has been successfully deployed in DeepSeek-R1 \cite{deepseekr1} and other systems, but its theoretical properties-particularly around \emph{credit assignment} and its links to Policy Gradient Methods remain poorly understood.

In this paper, we show that the GRPO gradient is fundamentally a policy gradient estimator, and derive it step by step from the policy gradient theorem. By making explicit the assumptions under which GRPO is formulated---output-only reward and a group-mean baseline---we reveal a structural consequence: \textbf{GRPO assigns identical credit to every token in a sequence}. This uniform credit assignment is not merely a simplification but a fundamental property that shapes the gradient's spectral characteristics.

\paragraph{Contributions.}
\begin{enumerate}
    \item We derive GRPO as a special case of the REINFORCE estimator with a group-mean baseline, making all approximations explicit (\cref{sec:derivation}).
    \item We prove that output-only reward induces a \emph{uniform credit assignment} where every token receives the same advantage, and characterize the resulting gradient sparsity (\cref{sec:credit}).
    \item We empirically demonstrate that the GRPO gradient matrix collapses to an effective rank $\approx 2$ independent of group size $R$, and formalize how the zero-sum constraint drives this behavior.
    \item We validate our theory empirically: SVD analysis of GRPO gradients on Nemotron-4B/GSM8K confirms effective rank $\approx 2$ for $R \in \{2,4,8\}$ (\cref{sec:experiments}).
    \item We also explain multi turn limitations of GRPO.
\end{enumerate}

\paragraph{Scope.} We focus on the policy gradient core of GRPO the advantage estimator and gradient structure deliberately setting aside clipping and KL regularization, which affect the optimization step but not the gradient signal's fundamental properties.

\section{From Policy Gradients to GRPO}
\label{sec:derivation}

We derive GRPO from first principles, following the chain: value functions $\to$ policy gradient theorem $\to$ REINFORCE with baseline $\to$ advantage estimation $\to$ GRPO.

\subsection{MDP Formulation for Text Generation}

We model autoregressive generation as an episodic MDP $\mathcal{M} = (\mathcal{S}, \mathcal{A}, P, r, \rho_0)$. A state $s_t$ is the concatenation of the prompt and tokens generated so far. An action $a_t \in \mathcal{A}$ (the vocabulary) appends a token. The transition $P$ is deterministic (concatenation). The initial distribution $\rho_0$ is the prompt distribution.

\begin{definition}[Value Functions]
For policy $\pi_\theta$ with parameters $\theta \in \R^d$:
\begin{align}
v^\pi(s) &= \E_\pi[G_t \mid S_t = s], \quad G_t = \sum_{k=0}^{T-t-1} \gamma^k r_{t+k+1}, \label{eq:value}\\
q^\pi(s,a) &= \E_\pi[G_t \mid S_t = s, A_t = a], \label{eq:qvalue}\\
A^\pi(s,a) &= q^\pi(s,a) - v^\pi(s). \label{eq:advantage}
\end{align}
\end{definition}

\subsection{Policy Gradient Theorem}

The objective $J(\theta) = \E_{\tau \sim \pi_\theta}[G_0]$ has gradient \cite{sutton1999policy}:
\begin{equation}
\nabla_\theta J(\theta) = \E_{\pi_\theta}\left[\sum_{t=0}^{T-1} \gamma^t q^{\pi_\theta}(S_t, A_t) \nabla_\theta \log \pi_\theta(A_t \mid S_t)\right].
\label{eq:pg_theorem}
\end{equation}

\subsection{The REINFORCE Estimator}

The REINFORCE algorithm \cite{williams1992simple} estimates the policy gradient by replacing the expected return with a single Monte Carlo sample. Given a trajectory $\tau = (s_0, a_0, r_1, \ldots, s_{T-1}, a_{T-1}, r_T)$, the estimator is:
\begin{equation}
\hat{g}_{\text{RF}} = \sum_{t=0}^{T-1}\gamma^t \hat{G}_t \nabla_\theta \log \pi_\theta(A_t \mid S_t),
\label{eq:reinforce_raw}
\end{equation}
where $\hat{G}_t = \sum_{k=0}^{T-t-1}\gamma^k r_{t+k+1}$ is the empirical return from time $t$. The key insight is that $\hat{G}_t$ is a Monte Carlo estimate of $q^\pi(S_t, A_t)$: by definition, $q^\pi(s,a) = \E_\pi[G_t \mid S_t = s, A_t = a]$, so a single rollout from state $s_t$ after taking action $a_t$ gives an unbiased (but high-variance) sample of the action-value function. This substitution replacing $q^\pi$ with the realized return from the trajectory is what makes REINFORCE a practical algorithm, at the cost of high variance.

\subsection{Variance Reduction via Baselines}

To reduce variance, we can subtract any state-dependent baseline $b(S_t)$ from the return. This does not introduce bias because:
\begin{equation}
\E_{\pi_\theta}\left[b(S_t) \nabla_\theta \log \pi_\theta(A_t \mid S_t)\right] = b(S_t)\sum_{a} \nabla_\theta \pi_\theta(a \mid S_t) = b(S_t) \cdot \nabla_\theta 1 = 0.
\label{eq:baseline_zero}
\end{equation}
Therefore we can write:
\begin{equation}
\hat{g}_{\text{RF}} = \sum_{t=0}^{T-1}\gamma^t (\hat{G}_t - b(S_t)) \nabla_\theta \log \pi_\theta(A_t \mid S_t).
\label{eq:reinforce}
\end{equation}
The variance-minimizing choice is $b^*(s) = v^\pi(s)$. To see why this yields the advantage, note that the return $\hat{G}_t$ can be decomposed via the Bellman equation. In the one-step (TD) form:
\begin{equation}
\hat{G}_t = r_{t+1} + \gamma \hat{G}_{t+1} = r_{t+1} + \gamma v^\pi(S_{t+1}).
\label{eq:td_decomp}
\end{equation}
Subtracting the baseline $b(S_t) = v^\pi(S_t)$:
\begin{equation}
\hat{G}_t - v^\pi(S_t) \approx r_{t+1} + \gamma v^\pi(S_{t+1}) - v^\pi(S_t) = \delta_t,
\label{eq:td_error}
\end{equation}
which is the temporal difference (TD) error. In expectation, this equals the advantage: $\E[\hat{G}_t - v^\pi(S_t) \mid S_t = s, A_t = a] = q^\pi(s,a) - v^\pi(s) = A^\pi(s,a)$. Thus the policy gradient becomes:
\begin{equation}
\nabla_\theta J(\theta) = \E\left[\sum_{t=0}^{T-1} \gamma^t A^\pi(S_t, A_t) \nabla_\theta \log \pi_\theta(A_t \mid S_t)\right].
\label{eq:pg_advantage}
\end{equation}
This is the advantage actor-critic gradient. The challenge is that computing $v^\pi(s)$ exactly requires a learned critic which is precisely what GRPO eliminates.

\subsection{Deriving GRPO}

GRPO arises from two specializations of the advantage-based policy gradient (\cref{eq:pg_advantage}).

\textbf{Specialization 1: Output-only reward (ORM).} In LLM training with an outcome reward model, reward is given only at the terminal state: $r(s_t, a_t) = 0$ for all $t < T-1$, and $r(s_{T-1}, a_{T-1}) = R_{\text{terminal}}$. Setting $\gamma = 1$, consider what happens to the advantage at each token position.

The return at any time step $t$ is:
\begin{equation}
\hat{G}_t = \sum_{k=t}^{T-1} r(s_k, a_k) = R_{\text{terminal}} \quad \forall\, t \in \{0, \ldots, T-1\},
\label{eq:output_only_return}
\end{equation}
since all intermediate rewards are zero. Now consider the value function: since the only reward comes at the end and $\gamma = 1$, the value of \emph{any} intermediate state $s_t$ along a trajectory is the expected terminal reward conditioned on being in that state. But crucially, once we condition on the full trajectory (as in REINFORCE), the realized return is the same $R_{\text{terminal}}$ regardless of which token position $t$ we evaluate.

This means the TD error at every position becomes:
\begin{equation}
\delta_t = r_{t+1} + v^\pi(s_{t+1}) - v^\pi(s_t) = 0 + v^\pi(s_{t+1}) - v^\pi(s_t) \quad
\end{equation}
and the full Monte Carlo advantage $\hat{G}_t - v^\pi(s_t) = R_{\text{terminal}} - v^\pi(s_t)$. Since all states share the same initial prompt $s_0$ and the value function under output-only reward satisfies $v^\pi(s_0) = \E_\pi[R_{\text{terminal}} \mid s_0]$, the advantage estimate for the entire trajectory collapses to a single scalar:
\begin{equation}
A_t = R_{\text{terminal}} - v^\pi(s_0) \quad \forall\, t.
\label{eq:uniform_advantage_derived}
\end{equation}
\textbf{Every token receives the same advantage.} The first token, the critical reasoning step, and filler tokens all get identical credit.

\textbf{Specialization 2: Group-mean baseline.} The remaining challenge is estimating $v^\pi(s_0)$. Instead of learning a critic, GRPO estimates it by sampling $R$ rollouts from the same prompt and computing the sample mean:
\begin{equation}
\hat{v}_R(s_0) = \frac{1}{R}\sum_{i=1}^R r_i,
\label{eq:grpo_baseline}
\end{equation}
where $r_i$ is the terminal reward of rollout $i$. This is a valid Monte Carlo estimate of $v^\pi(s_0) = \E_\pi[R_{\text{terminal}} \mid s_0]$.

Substituting both specializations into \cref{eq:reinforce}, the GRPO gradient for a single prompt is:
\begin{equation}
\boxed{\hat{g}_{\text{GRPO}} = \frac{1}{R}\sum_{i=1}^R \underbrace{(r_i - \hat{v}_R)}_{\hat{A}_i} \sum_{t=0}^{T_i-1} \nabla_\theta \log \pi_\theta(a_t^{(i)} \mid s_t^{(i)}).}
\label{eq:grpo_gradient} 
\end{equation}

This derivation makes explicit that GRPO is REINFORCE with a specific baseline choice, operating under the output-only reward assumption. The key consequence is immediate: $\hat{A}_i$ does not depend on $t$.

\textbf{Note:} The full GRPO objective also includes importance sampling clipping (analogous to PPO) and a KL divergence penalty against a reference policy. We omit these terms here as they affect the optimization step size but not the fundamental structure of the gradient signal---the credit assignment and rank properties we analyze depend only on the advantage estimator.

\section{Credit Assignment Failure}
\label{sec:credit}

We now formalize the credit assignment implications of \cref{eq:grpo_gradient}. Recall that in GRPO, we sample $R$ rollouts from a single prompt $s_0$ and compute one shared baseline $\hat{v}_R(s_0) = \frac{1}{R}\sum_{i=1}^R r_i$. The advantage for rollout $i$ is then $\hat{A}_i = r_i - \hat{v}_R(s_0)$---a scalar that measures how much better (or worse) rollout $i$ performed relative to the group. Note that $\hat{A}_i \neq 0$ in general; it is positive for above-average rollouts and negative for below-average ones.

The critical observation is what happens \emph{within} each rollout:

\begin{theorem}[Uniform Token-Level Credit]
\label{thm:uniform_credit}
Under output-only reward with $\gamma = 1$, the advantage assigned to every token position $t$ within rollout $i$ is the same scalar:
\begin{equation}
\hat{A}_i(t) = r_i - \hat{v}_R(s_0) = \hat{A}_i \quad \forall\, t \in \{0, \ldots, T_i - 1\}.
\end{equation}
That is, GRPO assigns a single rollout-level credit signal uniformly across all token positions. The per-token gradient contribution is:
\begin{equation}
\hat{g}_i(t) = \hat{A}_i \cdot \nabla_\theta \log \pi_\theta(a_t^{(i)} \mid s_t^{(i)}),
\end{equation}
where $\hat{A}_i$ depends on the group of $R$ rollouts but \textbf{not} on which token position $t$ within rollout $i$ we are at.
\end{theorem}

\begin{proof}
Under ORM, the return at any position $t$ within rollout $i$ is $G_t^{(i)} = R_{\text{terminal}}^{(i)} = r_i$ (since all intermediate rewards are zero and $\gamma = 1$). The baseline $\hat{v}_R(s_0) = \frac{1}{R}\sum_{j=1}^R r_j$ is computed once from the $R$ group rewards and shared across all positions. Therefore $\hat{A}_i(t) = G_t^{(i)} - \hat{v}_R(s_0) = r_i - \hat{v}_R(s_0) = \hat{A}_i$, independent of $t$.
\end{proof}

\begin{remark}[Credit Assignment Interpretation]
The advantage $\hat{A}_i$ is non-zero---it captures whether rollout $i$ is better or worse than the group average. However, this credit signal cannot distinguish \emph{which tokens} within rollout $i$ were responsible for the outcome. A rollout that succeeds because of a single correct reasoning step in the middle has every token---including irrelevant preamble and filler---reinforced equally. GRPO provides \emph{inter-rollout} credit (which rollouts are good) but no \emph{intra-rollout} credit (which tokens within a rollout are good).
\end{remark}

\subsection{Gradient Factorization}

Since $\hat{A}_i$ is the same for every token in rollout $i$, we can factor it out of the inner sum in \cref{eq:grpo_gradient}. Define the \emph{trajectory score} of rollout $i$ as the sum of all per-token log-probability gradients:
\begin{equation}
\psi_i = \sum_{t=0}^{T_i-1}\nabla_\theta \log \pi_\theta(a_t^{(i)} \mid s_t^{(i)}) \in \R^d.
\end{equation}
This is a single vector that represents the direction in parameter space that would increase the likelihood of the entire sequence $i$. The GRPO gradient then simplifies to:
\begin{equation}
\hat{g}_{\text{GRPO}} = \frac{1}{R}\sum_{i=1}^R \hat{A}_i \cdot \psi_i.
\label{eq:factorization}
\end{equation}
Each rollout contributes one scalar weight ($\hat{A}_i$, how good the rollout was) times one direction ($\psi_i$, how to make it more likely). If GRPO had token-level advantages, we could not factor this way---different tokens would pull in different directions with different magnitudes. The uniform credit is what enables this rank-reducing factorization.

\section{Structural Constraints on Advantages}
\label{sec:sparsity}

The factorized form \cref{eq:factorization} reveals that the GRPO gradient is entirely determined by the advantage vector $(\hat{A}_1, \ldots, \hat{A}_R)$ and the trajectory scores $(\psi_1, \ldots, \psi_R)$. We now characterize structural properties of the advantage vector that constrain the gradient.

\begin{theorem}[Zero-Sum and Correlation Structure]
\label{thm:sparsity}
The GRPO advantages satisfy:
\begin{enumerate}
\item[(i)] \textbf{Zero-sum:} $\sum_{i=1}^R \hat{A}_i = 0$. The advantages within a group always cancel---positive and negative rollouts exactly balance.
\item[(ii)] \textbf{Variance:} $\Var(\hat{A}_i) = \frac{R-1}{R}\sigma_r^2$, where $\sigma_r^2 = \Var_{\pi_\theta}(r \mid s_0)$ is the reward variance under the current policy.
\item[(iii)] \textbf{Negative correlation:} $\Cov(\hat{A}_i, \hat{A}_j) = -\frac{\sigma_r^2}{R}$ for $i \neq j$. Making one rollout's advantage more positive necessarily pushes others more negative.
\item[(iv)] \textbf{$\epsilon$-Effective rollouts:} Define a rollout $i$ as $\epsilon$-effective if $|\hat{A}_i| = |r_i - \hat{v}_R| > \epsilon$. The expected fraction of $\epsilon$-effective rollouts is bounded by:
\begin{equation}
\rho_\epsilon = \Prob\left(|\hat{A}_i| > \epsilon\right) \leq \frac{(R-1)\sigma_r^2}{R\epsilon^2}.
\label{eq:sparsity_bound}
\end{equation}
Only $\epsilon$-effective rollouts contribute meaningful gradient signal; the rest produce near-zero updates.
\end{enumerate}
\end{theorem}

\begin{proof}
(i) $\sum_i \hat{A}_i = \sum_i (r_i - \hat{v}_R) = \sum_i r_i - R \cdot \frac{1}{R}\sum_i r_i = 0$.

(ii) Write $\hat{A}_i = r_i - \frac{1}{R}\sum_j r_j = \frac{R-1}{R}r_i - \frac{1}{R}\sum_{j\neq i}r_j$. Since the $r_j$ are i.i.d.:
$\Var(\hat{A}_i) = \left(\frac{R-1}{R}\right)^2\sigma_r^2 + \frac{R-1}{R^2}\sigma_r^2 = \frac{R-1}{R}\sigma_r^2$.

(iii) From $\Var(\sum_i \hat{A}_i) = 0$: $R\Var(\hat{A}_i) + R(R-1)\Cov(\hat{A}_i, \hat{A}_j) = 0$, giving $\Cov(\hat{A}_i, \hat{A}_j) = -\sigma_r^2/R$.

(iv) By Chebyshev's inequality: $\Prob(|\hat{A}_i| > \epsilon) \leq \Var(\hat{A}_i)/\epsilon^2 = \frac{R-1}{R}\sigma_r^2 / \epsilon^2$.

\end{proof}

\begin{corollary}[Gradient Signal Decay]
\label{cor:progressive}
The magnitude of the gradient $\|\hat{g}_{\text{GRPO}}\|$ scales with $\sigma_r$---the reward standard deviation under the current policy. As training progresses and the policy converges (i.e., it gets most problems right or most wrong for a given prompt), $\sigma_r \to 0$ and the advantages $\hat{A}_i \to 0$ for all $i$ simultaneously. When all rollouts from a prompt receive the same reward, $\hat{A}_i = 0$ exactly and the prompt contributes zero gradient.
\end{corollary}

The zero-sum constraint means GRPO can only learn from prompts where the $R$ rollouts produce \emph{diverse} rewards---some correct, some incorrect. Prompts that are too easy (all correct) or too hard (all wrong) yield $\hat{A}_i = 0$ for every rollout and contribute nothing to learning. This makes GRPO a natural \emph{curriculum}: it automatically focuses on prompts at the boundary of the policy's capability, where reward variance is highest.

\section{Sparse Gradient Updates and Rank Collapse}
\label{sec:rank}

We now discuss how the GRPO gradient structure encourages sparse parameter updates and empirically leads to low-rank gradients.

\subsection{Why Updates Are Sparse}

From \cref{eq:factorization}, the gradient is $\hat{g} = \frac{1}{R}\sum_i \hat{A}_i \psi_i$. Two structural properties conspire to make this sparse:

\textbf{(1) Zero-sum advantages concentrate signal.} Since $\sum_i \hat{A}_i = 0$, the gradient is determined entirely by the \emph{contrast} between rollouts. Consider binary reward ($r \in \{0, 1\}$) as in GSM8K: if $k$ of $R$ rollouts are correct, then the correct rollouts have $\hat{A}_i = 1 - k/R$ and incorrect ones have $\hat{A}_i = -k/R$. When $k$ is close to $0$ or $R$ (easy/hard prompts), the advantages become small and the gradient vanishes. Only prompts where some rollouts succeed and some fail produce meaningful signal.

\textbf{(2) Uniform credit concentrates in few directions.} Because $\hat{A}_i$ is a single scalar for all tokens in rollout $i$, the gradient for each rollout points in one direction $\psi_i$. If trajectory scores are similar (e.g., rollouts from the same prompt share syntactic structure), the $R$ directions $\psi_i$ cluster. We hypothesize that this clustering, combined with the zero-sum advantages, is what empirically concentrates the gradient into a low-dimensional (often rank-2) subspace.

\section{Multi-Turn Limitation}
\label{sec:multiturn}

The derivation in \cref{sec:derivation} reveals why GRPO is inherently a single-turn method. Recall the standard advantage at state $s_t$:
\begin{equation}
A^\pi(s_t, a_t) = q^\pi(s_t, a_t) - v^\pi(s_t) = r_{t+1} + \gamma v^\pi(s_{t+1}) - v^\pi(s_t).
\label{eq:true_advantage}
\end{equation}
This requires bootstrapping from the \emph{next state's value} $v^\pi(s_{t+1})$---which is what a learned critic provides in PPO.

GRPO has no critic. It uses the group-mean baseline $\hat{v}_R(s_0)$, which estimates the value of the \emph{initial state} $s_0$ only. In the single-turn ORM setting this is fine: with $\gamma = 1$ and $r_t = 0$ for $t < T-1$, the advantage at every position reduces to $R_{\text{terminal}} - v^\pi(s_0)$, and $\hat{v}_R(s_0)$ is an unbiased estimate of $v^\pi(s_0)$.

But in multi-turn settings with per-step rewards $r_1, \ldots, r_H$, the true advantage at stage $h$ requires $v^\pi(s_h)$---the value of the \emph{current} state, not the initial state. GRPO still subtracts $\hat{v}_R(s_0)$ from the total return:
\begin{equation}
\hat{A}_{\text{GRPO}}(h) = \sum_{h'=1}^H r_{h'} - \hat{v}_R(s_0).
\end{equation}
This conflates past rewards (already received, not actionable) with future rewards (what the agent can still influence), introducing bias:
\begin{equation}
\Bias_h = \sum_{h'=1}^{h-1} r_{h'} + v^\pi(s_h) - v^\pi(s_0),
\end{equation}
which grows linearly with stage $h$. The first term is accumulated past reward; the second is the value mismatch from using $v^\pi(s_0)$ instead of $v^\pi(s_h)$.

\textbf{In short:} GRPO cannot perform the bootstrapping step $\gamma v^\pi(s_{t+1})$ because it has no value function. This makes it exact for single-turn generation (where bootstrapping is unnecessary) but introduces growing bias for multi-turn problems where the value of states changes along the trajectory.

\section{Experiments}
\label{sec:experiments}

We validate our theoretical predictions empirically using GRPO training of Nemotron-4B on GSM8K.

\subsection{Setup}

We train Nemotron-4B \cite{nvidia2024nemotron} with GRPO on GSM8K (mathematical reasoning, binary correctness reward). We vary the group size $R \in \{2, 4, 8\}$ across 3 random seeds. Using a custom gradient SVD tracker integrated into our FSDP training pipeline, we explicitly capture the fully reduced flattened gradients ($\nabla_\theta J$) for tracked linear layers (e.g., \texttt{q\_proj}, \texttt{v\_proj}) every 50 optimizer steps. For each tracked layer at each checkpoint, we stack the per-rollout gradient contributions into a matrix $G \in \R^{R \times d}$ and compute its truncated SVD.

\textbf{Metrics:}
\begin{itemize}
\item \textbf{Effective rank}: $r_{\text{eff}} = \left(\sum_i \sigma_i\right)^2 / \sum_i \sigma_i^2$
\item \textbf{Top-1 fraction}: $\sigma_1^2 / \sum_i \sigma_i^2$ (variance explained by first component)
\item \textbf{Reward std}: $\sigma_r(k)$ at training step $k$
\end{itemize}

\subsection{Results}

\begin{figure}[t]
\centering
\includegraphics[width=\columnwidth]{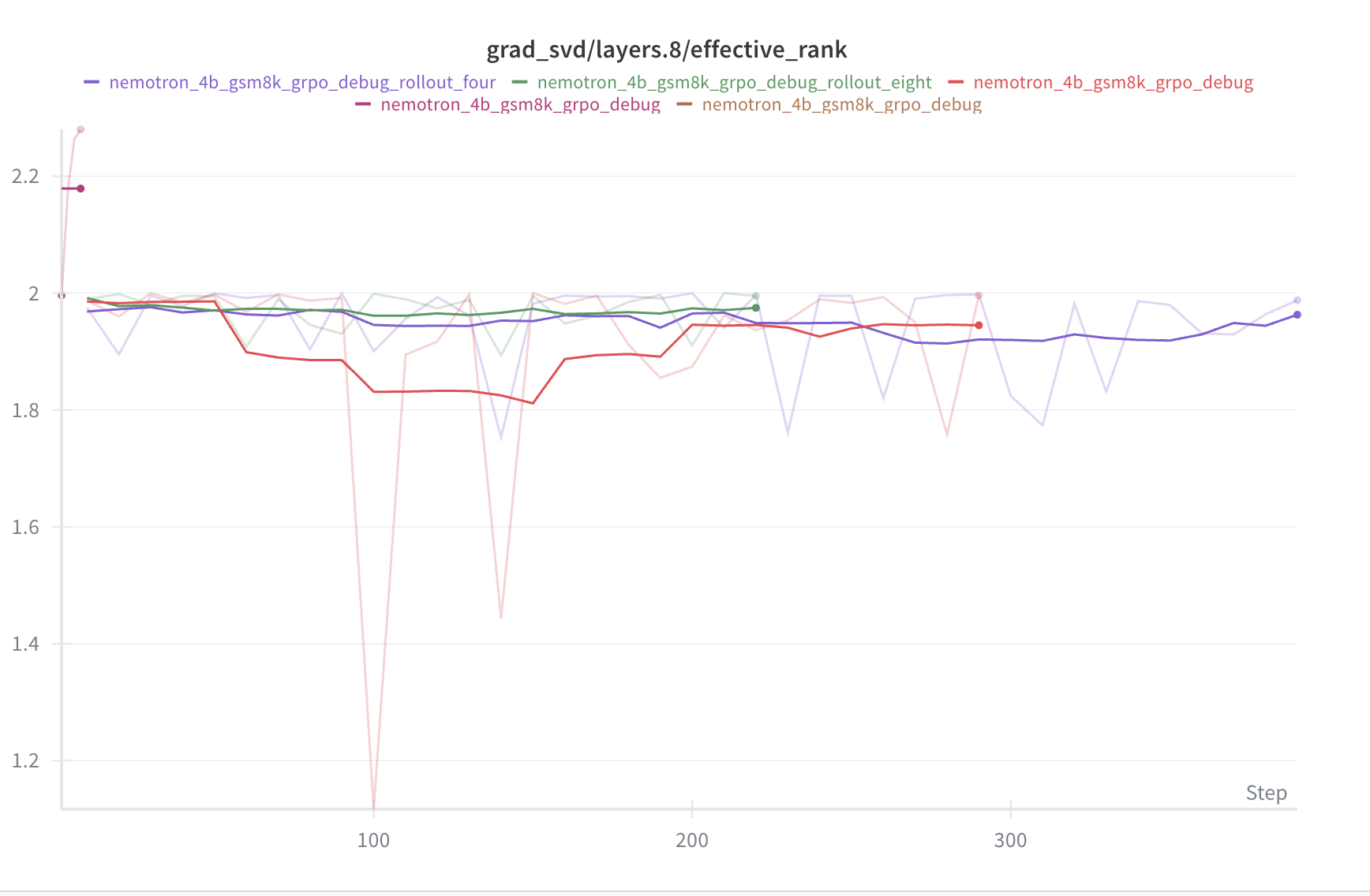}
\caption{Effective rank of the GRPO gradient matrix across training steps for $R \in \{2, 4, 8\}$. Despite the maximum possible rank being higher, effective rank remains $\approx 2$ throughout training.}
\label{fig:effective_rank}
\end{figure}

\begin{figure}[t]
\centering
\includegraphics[width=\columnwidth]{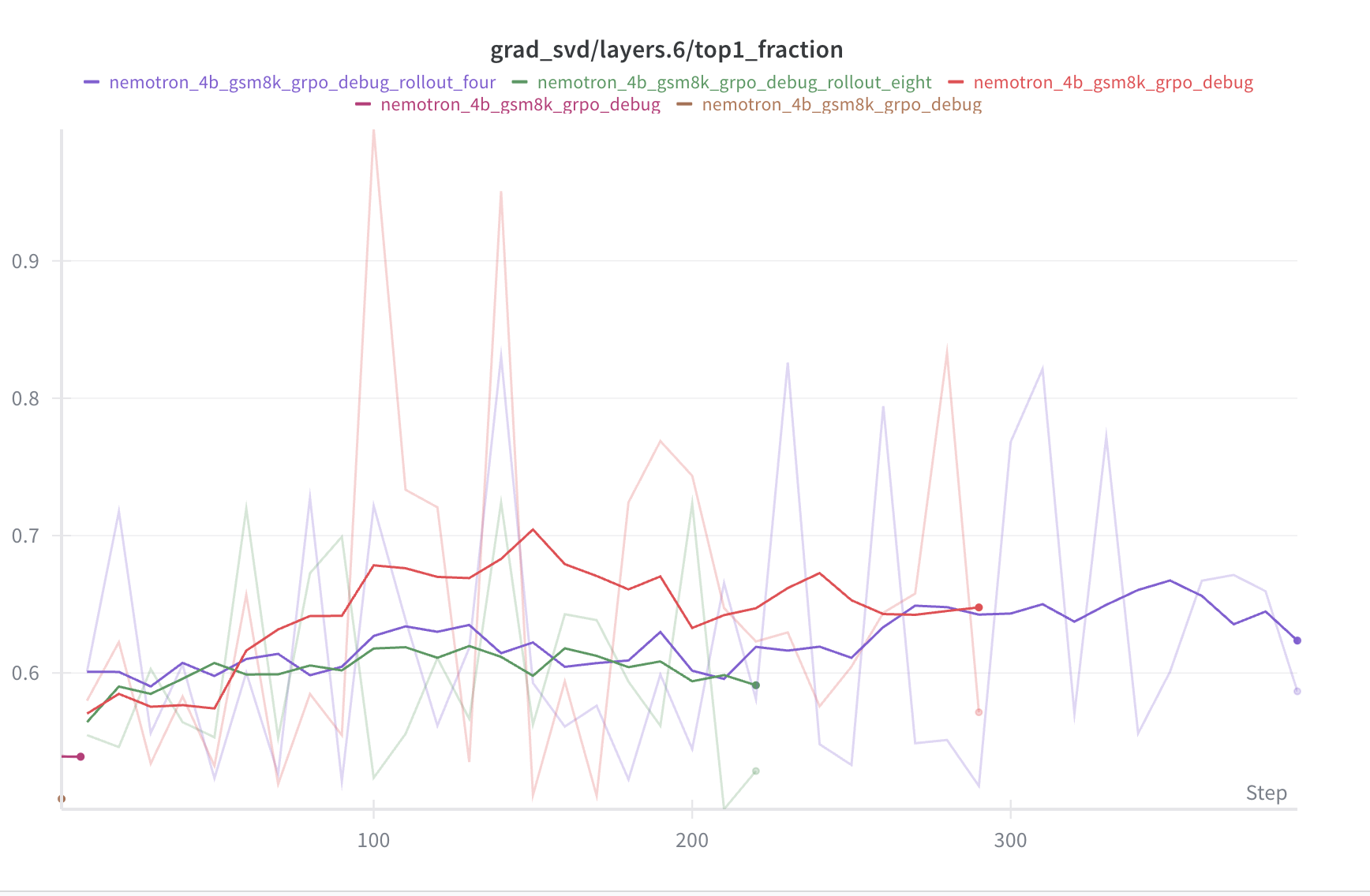}
\caption{Top-1 singular value fraction ($\sigma_1^2 / \sum_i \sigma_i^2$) across training. $R=4$ shows the highest concentration, while $R=8$ diffuses energy across more components. This non-monotonic pattern is explained by the interplay between advantage concentration and trajectory score diversity.}
\label{fig:top1_fraction}
\end{figure}

\begin{figure}[t]
\centering
\includegraphics[width=\columnwidth]{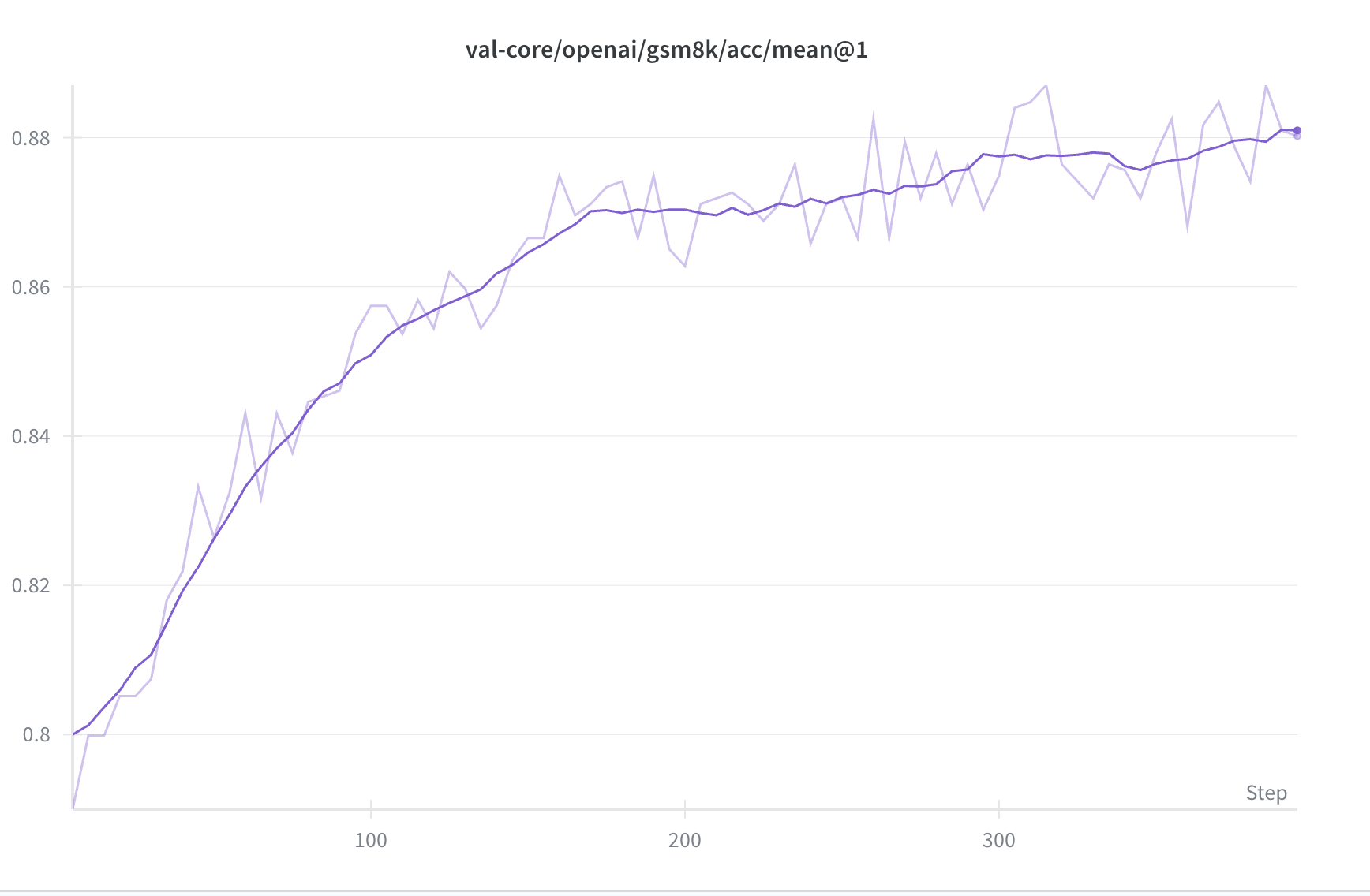}
\caption{Training accuracy on GSM8K for different group sizes. Despite rank-2 gradient structure for all $R$, larger groups still improve learning speed by providing better estimates of the advantage direction $\bar{\psi}_+ - \bar{\psi}_-$.}
\label{fig:accuracy}
\end{figure}

\paragraph{Effective rank $\approx 2$ for all $R$.}
Across all group sizes, the effective rank of the gradient matrix remains approximately $2$ (\cref{tab:results}). This holds despite the theoretical maximum rank being larger for groups $R \geq 2$.

\begin{table}[t]
\caption{Spectral properties of GRPO gradient matrices (Nemotron-4B, GSM8K). Effective rank stays $\approx 2$ while max possible rank grows with $R$.}
\label{tab:results}
\begin{center}
\begin{small}
\begin{sc}
\begin{tabular}{lcccc}
\toprule
$R$ & Max Rank & Eff. Rank & Top-1 Frac. \\
\midrule
2 & 2\textsuperscript{*} & $\sim$2 & 0.6--0.7 \\
4 & 3 & $\sim$2 & 0.7--0.8 \\
8 & 7 & $\sim$2 & 0.5--0.6 \\
\bottomrule
\end{tabular}
\end{sc}
\end{small}
\end{center}
\raggedright
\scriptsize{\textsuperscript{*}For $R=2$, while advantages are zero-sum ($\hat{A}_1 = -\hat{A}_2$), the trajectory scores $\psi_1$ and $\psi_2$ are linearly independent, making the maximum possible rank 2.}
\vskip -0.1in
\end{table}

\paragraph{Interpretation.}
The rank-2 structure means that \textbf{each GRPO gradient step moves parameters along only $\sim$2 directions}, regardless of whether 2 or 8 rollouts are computed. The additional rollouts for $R > 2$ refine the \emph{magnitude} and \emph{direction} of these two components but do not add new dimensions of movement. This has direct implications for training efficiency: increasing $R$ provides diminishing returns in terms of gradient informativeness.

\section{Discussion and Future Work}
\label{sec:discussion}

\paragraph{Connection to memorization vs.\ generalization.}
The uniform credit assignment (\cref{thm:uniform_credit}) means GRPO reinforces entire sequences indiscriminately. From the memorization/generalization perspective, this biases toward \emph{template-level learning}: sequences that consistently achieve high reward are reinforced as holistic units, rather than having their compositional structure decomposed. This may explain why GRPO-trained models can reproduce successful reasoning patterns but struggle to compose them in novel ways.

\paragraph{Implications for multi-step reasoning.}
The multi-turn bias  and rank-2 collapse together suggest that GRPO is poorly suited for learning \emph{which step} in a reasoning chain matters. Process reward models (PRMs) that provide per-step credit are a natural remedy, but our analysis shows that even with dense reward, the group-relative baseline introduces structural limitations (negative correlations, zero-sum constraints) that constrain the gradient's expressiveness.

\paragraph{Future directions.}
(1) \emph{Adaptive baselines}: Can leave-one-out (RLOO) or learned token-level baselines break the rank-2 bottleneck?
(2) \emph{Rank-aware training}: If effective rank predicts training saturation, can we use it as a signal to increase $R$ or switch to denser rewards?
(3) \emph{Progressive rank collapse over training}: Our theory predicts effective rank should decrease as $\sigma_r \to 0$; longitudinal measurement would validate \cref{cor:progressive}.

\section{Related Work}
\label{sec:related}

\textbf{Policy gradients and baselines.} The policy gradient theorem \cite{sutton1999policy}, REINFORCE \cite{williams1992simple}, and optimal baselines \cite{weaver2001optimal} provide the foundation. GAE \cite{schulman2015high} offers a bias-variance tradeoff via temporal-difference mixing.

\textbf{GRPO and RLHF.} GRPO was introduced in DeepSeek-Math \cite{shao2024deepseekmath} and used in DeepSeek-R1 \cite{deepseekr1}. RLOO \cite{ahmadian2024back} uses a leave-one-out baseline that eliminates the $O(1/R)$ bias. DPO \cite{rafailov2023direct} avoids RL entirely.

\textbf{Credit assignment in LLMs.} Process reward models \cite{lightman2023lets} address token-level credit. Recent work on advantage-weighted regression and return-conditioned training explores alternative credit mechanisms.

\textbf{Gradient rank in deep learning.} Low-rank gradient structure has been observed in supervised learning \cite{li2018measuring} and exploited in efficient fine-tuning (LoRA, \citealp{hu2021lora}). Our work identifies a \emph{theoretical} reason for low rank specific to GRPO's policy gradient structure.

\section{Conclusion}

We have shown that GRPO, derived from first principles of the policy gradient theorem, suffers from a fundamental credit assignment limitation: output-only reward forces identical advantage across all tokens, collapsing the gradient into a rank-2 structure independent of group size. This is validated empirically on Nemotron-4B/GSM8K. Our results formalize when GRPO's simplicity is theoretically justified (single-turn, high initial $\sigma_r$) and where it breaks down (multi-step reasoning requiring fine-grained credit). Breaking the rank-2 bottleneck---via denser rewards, token-level baselines, or hybrid approaches---is an important direction for scaling RLHF to complex reasoning tasks.

\bibliography{workshop_refs}
\bibliographystyle{icml2026_fogen}

\newpage
\appendix
\onecolumn
\section{Proofs}
\label{app:proofs}

\subsection{Proof of \cref{thm:sparsity} (Full Details)}

\textbf{Zero-sum:} $\sum_{i=1}^R \hat{A}_i = \sum_i (r_i - \hat{v}_R) = \sum_i r_i - R\hat{v}_R = \sum_i r_i - \sum_i r_i = 0.$

\textbf{Variance:} Write $\hat{A}_i = r_i - \frac{1}{R}\sum_j r_j = \frac{R-1}{R}r_i - \frac{1}{R}\sum_{j\neq i}r_j$. Then:
\begin{align}
\Var(\hat{A}_i) &= \left(\frac{R-1}{R}\right)^2\Var(r_i) + \frac{1}{R^2}\sum_{j\neq i}\Var(r_j) \\
&= \frac{(R-1)^2}{R^2}\sigma_r^2 + \frac{R-1}{R^2}\sigma_r^2 = \frac{(R-1)^2 + (R-1)}{R^2}\sigma_r^2 = \frac{(R-1)R}{R^2}\sigma_r^2 = \frac{R-1}{R}\sigma_r^2.
\end{align}

\textbf{Covariance:} From $\Var(\sum_i \hat{A}_i) = 0$:
$$R\Var(\hat{A}_i) + R(R-1)\Cov(\hat{A}_i, \hat{A}_j) = 0 \implies \Cov(\hat{A}_i, \hat{A}_j) = -\frac{\Var(\hat{A}_i)}{R-1} = -\frac{\sigma_r^2}{R}.$$

\textbf{Sparsity bound:} $\Var(r_i - \hat{v}_R) = \Var(\hat{A}_i) = \frac{R-1}{R}\sigma_r^2$. By Chebyshev: $\Prob(|\hat{A}_i| > \epsilon) \leq \frac{(R-1)\sigma_r^2}{R\epsilon^2}$.

\subsection{Hoeffding Inequality Analysis}

The baseline $\hat{v}_R(s_0) = \frac{1}{R}\sum_{i=1}^R r_i$ is the arithmetic mean of $R$ i.i.d.\ random variables with $r_i \in [0,1]$ and $\E[r_i] = v^\pi(s_0)$.

By Hoeffding's inequality for bounded i.i.d.\ random variables:
$$\Prob\left(\left|\frac{1}{R}\sum_{i=1}^R r_i - \E[r_i]\right| > \epsilon\right) \leq 2\exp\left(\frac{-2R^2\epsilon^2}{\sum_{i=1}^R (b_i - a_i)^2}\right) = 2\exp\left(\frac{-2R^2\epsilon^2}{R \cdot 1^2}\right) = 2\exp(-2R\epsilon^2).$$

\textbf{Interpretation by group size:}
\begin{center}
\begin{tabular}{lccc}
\toprule
$R$ & $\Prob(|\hat{v}_R - v^\pi| > 0.3)$ & $\Prob(|\hat{v}_R - v^\pi| > 0.2)$ & $\Prob(|\hat{v}_R - v^\pi| > 0.1)$ \\
\midrule
2 & $\leq 1.34$ (vacuous) & $\leq 1.70$ (vacuous) & $\leq 1.92$ (vacuous) \\
4 & $\leq 0.93$ (vacuous) & $\leq 1.44$ (vacuous) & $\leq 1.85$ (vacuous) \\
8 & $\leq 0.44$ & $\leq 1.05$ (vacuous) & $\leq 1.71$ (vacuous) \\
16 & $\leq 0.10$ & $\leq 0.55$ & $\leq 1.47$ (vacuous) \\
\bottomrule
\end{tabular}
\end{center}

The bound becomes non-trivial only for larger $R$ or larger $\epsilon$. This reflects the fact that with few rollouts ($R = 2,4$), the baseline has substantial variance ($\sigma_r^2/R$), and the advantages are noisy estimates of the true $r_i - v^\pi(s_0)$.

\textbf{Connection to advantage quality:} The advantage error for rollout $i$ is:
$$\hat{A}_i - (r_i - v^\pi(s_0)) = v^\pi(s_0) - \hat{v}_R(s_0).$$
So this directly bounds the advantage estimation error: with probability $\geq 1 - 2e^{-2R\epsilon^2}$, every rollout's advantage is within $\epsilon$ of its true value $r_i - v^\pi(s_0)$.

\subsection{Multi-Turn Bias Derivation}

In multi-turn with $H$ stages, the true advantage at stage $h$ is:
$$A^\pi(s_h, a_h) = \E_\pi\left[\sum_{h'=h}^H r_{h'} \mid s_h, a_h\right] - v^\pi(s_h).$$

GRPO, lacking a value function for intermediate states, uses the same estimate for all stages:
$$\hat{A}_{\text{GRPO}}(h) = \underbrace{\sum_{h'=1}^H r_{h'}}_{\text{total return}} - \underbrace{\hat{v}_R(s_0)}_{\text{baseline from initial state}}.$$

The key issue: GRPO subtracts $\hat{v}_R(s_0) \approx v^\pi(s_0)$ when it should subtract $v^\pi(s_h)$. Taking expectations conditioned on $(s_h, a_h)$:
\begin{align}
\E[\hat{A}_{\text{GRPO}}(h) \mid s_h, a_h] &= \underbrace{\sum_{h'=1}^{h-1}r_{h'}}_{\text{past rewards (fixed)}} + \underbrace{\E\left[\sum_{h'=h}^H r_{h'} \mid s_h, a_h\right]}_{= q^\pi(s_h, a_h)} - v^\pi(s_0) \\
&= \sum_{h'=1}^{h-1}r_{h'} + A^\pi(s_h,a_h) + v^\pi(s_h) - v^\pi(s_0).
\end{align}

Therefore:
$$\Bias_h = \E[\hat{A}_{\text{GRPO}}(h)] - A^\pi(s_h, a_h) = \underbrace{\sum_{h'=1}^{h-1}r_{h'}}_{\text{past reward leakage}} + \underbrace{v^\pi(s_h) - v^\pi(s_0)}_{\text{value mismatch}}.$$

The bias has two sources:
\begin{enumerate}
\item \textbf{Past reward leakage}: Rewards from stages $1, \ldots, h-1$ are included in the return but are not attributable to action $a_h$. This grows as $h$ increases.
\item \textbf{Value mismatch}: GRPO bootstraps from $v^\pi(s_0)$ instead of $v^\pi(s_h)$. In the standard advantage (\cref{eq:true_advantage}), we subtract $v^\pi(s_t)$---the value of the current state. Without a critic, GRPO cannot compute $v^\pi(s_h)$ for $h > 0$.
\end{enumerate}

Bounding: $|\Bias_h| \leq (h-1)r_{\max} + |v^\pi(s_h) - v^\pi(s_0)| \leq (h-1)r_{\max} + 2V_{\max}$.

For single-turn ($H = 1$): $\Bias_1 = 0$ (no past rewards, no value mismatch). GRPO is exact.

\end{document}